\newif\ifdraft \drafttrue

\documentclass[10pt,letterpaper]{article}
\usepackage{cleveref}
\usepackage[margin=1in]{geometry}
\usepackage{amsmath,amsfonts,graphpap,amscd,mathrsfs,graphicx,lscape}
\usepackage{epsfig,amssymb,amstext,xspace}
\usepackage{float}	
\usepackage{hyperref}

\usepackage{color}              
\usepackage{epsfig}

\usepackage[]{color-edits}
\addauthor{sw}{blue}
\addauthor{vs}{red}

\usepackage{amsthm}

\newcommand{\E}{\mathbb{E}}

\newtheorem{theorem}{Theorem}
\newtheorem{lemma}[theorem]{Lemma}

\newtheorem{assumption}[theorem]{Assumption}

\newtheorem{defn}{Definition}

\def\squareforqed{\hbox{\rule{2.5mm}{2.5mm}}}

\def\QED{\ifmmode\squareforqed 
  \else{\nobreak\hfil   
    \penalty50                 
    \hskip1em                  
    \null                      
    \nobreak                   
    \hfil                      
    \squareforqed              
    \parfillskip=0pt           
    \finalhyphendemerits=0     
    \endgraf}                  
  \fi}

\def\blksquare{\rule{2mm}{2mm}}
\def\qedsymbol{\blksquare}
\newcommand{\bg}[1]{\medskip\noindent{\bf #1}}
\newcommand{\ed}{{\hfill\qedsymbol}\medskip}

\newcommand{\comment}[1]{}
 {}

\newcommand{\junk}[1]{}

\newlength{\tmp} \newlength{\lpsx} \newlength{\lpsy} \newlength{\upsx} \newlength{\upsy}

\newcommand{\Omit}[1]{}

\title{A Proof of Orthogonal Double Machine Learning with $Z$-Estimators}
\author{Vasilis Syrgkanis\\
Microsoft Research}

\begin{document}
\maketitle    
\date{}

\begin{abstract}
We consider two stage estimation with a non-parametric first stage and a generalized method of moments second stage, in a simpler setting than \cite{Chernozhukov2016}. We give an alternative proof of the theorem given in Chernozhukov et al. \cite{Chernozhukov2016} that orthogonal second stage moments, sample splitting and $n^{1/4}$-consistency of the first stage, imply $\sqrt{n}$-consistency and asymptotic normality of second stage estimates. Our proof is for a variant of their estimator, which is based on the empirical version of the moment condition (Z-estimator), rather than a minimization of a norm of the empirical vector of moments (M-estimator). This note is meant primarily for expository purposes, rather than as a new technical contribution.
\end{abstract}
\section{Two-Stage Estimation}

Suppose we have a model which predicts the following set of moment conditions:
\begin{equation}
\E[m(Z,\theta_0, h_0(X))] = 0
\end{equation}
where $\theta_0\in R^d$ is a finite dimensional parameter of interest, $h_0: S\rightarrow R^\ell$ is a nuisance function we do not know, $Z$ are the observed data which are drawn from some distribution and $X\in S$ is a subvector of the observed data.

We want to understand the asymptotic properties of the following two-stage estimation process:
\begin{enumerate}
\item First stage. Estimate $h_0(\cdot)$ from an auxiliary data set (e.g. running some non-parametric regresssion) yielding an estimate $\hat{h}$. 
\item Second stage. Use the first stage estimate $\hat{h}$ and compute an estimate $\hat{\theta}$ of $\theta_0$ from an empirical version of the moment condition: i.e. 
\begin{equation}
\hat{\theta} \text{ solves }: \frac{1}{n} \sum_{t=1}^n m(Z_t, \hat{\theta}, \hat{h}(X_t)) =0
\end{equation}
\end{enumerate} 

The question we want to ask is: is $\hat{\theta}$ $\sqrt{n}$-consistent. More formally, is it true that:
\begin{equation}
\sqrt{n} (\hat{\theta}-\theta_0) \rightarrow N(0,\Sigma)
\end{equation}
for some constant co-variance matrix $\Sigma$. We will assume that the moment conditions that we use satisfy the following orthogonality property:
\begin{defn}[Orthogonality] For any fixed estimate $\hat{h}$ that can be the outcome of the first stage estimation, the moment conditions are orthogonal if: 
\begin{equation}
\E\left[\nabla_{\gamma} m(Z,\theta_0, h_0(X)) \cdot (\hat{h}(X)-h_0(X))\right] = 0
\end{equation}
where $\nabla_\gamma m(\cdot,\cdot, \cdot)$ denotes the gradient of $m$ with respect to its third argument.
\end{defn}

\section{Orthogonality Implies Root-$n$ Consistency}

\begin{assumption}\label{ass:main}
We will make the following regularity assumptions:
\begin{itemize}
\item {\bf Rate of First Stage.} The first stage estimation is $n^{-1/4}$-consistent in the squared mean-square-error sense, i.e.
\begin{equation}
n^{1/2} E_{X}\left[\|\hat{h}(X)-h_0(X)\|^2\right] \rightarrow_p 0
\end{equation}
where the convergence in probability statement is with respect to the auxiliary data set
\item {\bf Regularity of First Stage.} The first stage estimate and the nuisance function are uniformly bounded by a constant, i.e.: $\|\hat{h}(x)\|, \|h_0(x)\| \leq C$ for all $x\in S$.
\item {\bf Regularity of Moments.} The following smoothness conditions hold for the moments
\begin{enumerate}
\item For any $z, x, \gamma$ the function $m(z, \theta, \gamma)$ is continuous in $\theta$. Also $m(z, \theta, \gamma)\leq d(z)$ and $\E[d(Z)]<\infty$. 
\item Similarly, the same conditions hold for $\nabla_\theta m(z, \theta, \gamma)$. 
\item $\E\left[\nabla_\theta m(z,\theta_0,h_0(x))\right]$ is non-singular. 
\item the Hessian $\nabla_{\gamma\gamma} m(z,\theta,\gamma)$ has the largest eigenvalue bounded by some constant $\lambda$ uniformly for all $\theta$ and $\gamma$. 
\item the derivative $\nabla_\gamma m(z,\theta,\gamma)$ has norm, uniformly bounded by $\sigma$
\end{enumerate}
\end{itemize}
\end{assumption}

\begin{theorem} Under Assumption \ref{ass:main} and assuming that $\hat{\theta}$ is consistent, if the moment conditions satisfy the orthogonality property then $\hat{\theta}$ is also $\sqrt{n}$-consistent and asymptotically normal.
\end{theorem}
\begin{proof}
By doing a first-order Taylor expansion of the empirical moment condition around $\theta_0$ and by the mean value theorem, we have:
\begin{equation}
\sqrt{n}(\hat{\theta}-\theta_0) =\underbrace{\left[ \frac{1}{n} \sum_{t=1}^n
\nabla_\theta m(Z_t, \tilde{\theta}, \hat{h}(X_t))\right]^{-1}}_{A} \underbrace{\frac{1}{\sqrt{n}} \sum_{t=1}^n m(Z_t,\theta_0, \hat{h}(X_t))}_{B}
\end{equation}
where $\tilde{\theta}$ is convex combination of $\theta_0$ and $\hat{\theta}$. We will show that $A$ converges in probability to a constant $J^{-1}$ and that $B$ converges in distribution to a normal $N(0,V)$, for some constant co-variance matrix $V$. Then the theorem follows by invoking Slutzky's theorem, which shows convergence in distribution to $N(0,J^{-1} V)$.

\paragraph{Convergence of $A$ to inverse derivative.} By the regularity of the moments, we have a uniform law of large numbers for the quantity $\frac{1}{n} \sum_{t=1}^n
\nabla_\theta m(Z_t, \theta, \hat{h}(X))$, i.e.:
\begin{equation}
\sup_{\theta\in \Theta} \left\|\frac{1}{n} \sum_{t=1}^n
\nabla_\theta m(Z_t, \theta, \hat{h}(X))- \E[\nabla_\theta m(Z, \theta, \hat{h}(X))]\right\| \rightarrow_p 0
\end{equation}
Since $\hat{\theta}$ is consistent, we also have that $\tilde{\theta}$ is consistent, i.e. $\tilde{\theta}\rightarrow_p \theta$. Combining the latter two properties, we get that conditional on the auxiliary data set:
\begin{equation}
\frac{1}{n} \sum_{t=1}^n \nabla_\theta m(Z_t, \tilde{\theta}, \hat{h}(X))
\rightarrow \E\left[\nabla_\theta m(Z,\theta_0, \hat{h}(X)\right]
\end{equation}
Moreover, since $\hat{h}$ is consistent we get that:
\begin{equation}
\frac{1}{n} \sum_{t=1}^n \nabla_\theta m(Z_t, \tilde{\theta}, \hat{h}(X))
\rightarrow \E\left[\nabla_\theta m(Z,\theta_0, h_0(X)\right]
\end{equation}
Since the matrix $\E\left[\nabla_\theta m(z,\theta_0,h_0(x))\right]$ is non-singular, by continuity of the inverse we get:
\begin{equation}
\left[\frac{1}{n} \sum_{t=1}^n \nabla_\theta m(Z_t, \tilde{\theta}, \hat{h}(X))\right]^{-1}
\rightarrow \left[\E\left[\nabla_\theta m(Z,\theta_0, h_0(X))\right]\right]^{-1} = J^{-1}
\end{equation}

\paragraph{Asymptotic normality of $B$.} To argue asymptotic normality of $B$ we take a second-order Taylor expansion of $B$ around $h_0(X_t)$ for each $X_t$: 
\begin{multline}
B = \underbrace{\frac{1}{\sqrt{n}} \sum_{t=1}^n m(Z_t, \theta_0, h_0(X_t))}_{C} +
\underbrace{\frac{1}{\sqrt{n}} \sum_{t=1}^n \nabla_{\gamma} m(Z_t, \theta_0, h_0(X_t)) \cdot \left(\hat{h}(X_t) - h_0(X_t)\right)}_{D}\\ + \underbrace{\frac{1}{2\sqrt{n}} \sum_{t=1}^n \left(\hat{h}(X_t) - h_0(X_t)\right)^T \nabla_{\gamma\gamma} m(Z_t,\theta_0, \tilde{h}(X_t)) \cdot \left(\hat{h}(X_t)-h_0(X_t)\right)}_{E}
\end{multline}
First we observe that $C$ is the sum of $n$ i.i.d. random variables, divided by $\sqrt{n}$. Thus by the Central Limit Theorem, we get that $C\rightarrow N(0,V)$, for some constant co-variance matrix $V$. Then we conclude by showing that $D, E\rightarrow_p 0$.

Second we argue that $n^{1/4}$ consistency of the first stage, implies that $E\rightarrow_p 0$. Since $\nabla_{\gamma\gamma} m(z,\theta,\gamma)$ has a largest eigenvalue uniformly bounded by $\lambda^*$, we have that the quantity $E$ is bounded by 
\begin{equation}
|E|\leq \frac{\lambda^*}{2} \sqrt{n} \left(\frac{1}{n} \sum_{t=1}^n \|\hat{h}(X_t)-h_0(X_t)\|^2 \right)
\end{equation} 
Fixing the auxiliary data set, the quantity $\frac{1}{n} \sum_{t=1}^n \|\hat{h}(X_t)-h_0(X_t)\|^2$ converges to $\E[\|\hat{h}(X_t)-h_0(X_t)\|^2]$. Subsequently by $n^{1/4}$-consistency of the first stage, and regularity of the first stage, we get that $E\rightarrow_p 0$.

Finally, we argue that orthogonality implies that $D\rightarrow_p 0$. We show that both the mean and the trace of the co-variance of $D$ converge to $0$. The mean conditional on the auxiliary data set is:
\begin{align}
E[D~|~\hat{h}] = \sqrt{n} E\left[\nabla_{\gamma} m(Z, \theta_0, h_0(X)) \cdot \left(\hat{h}(X) - h_0(X)\right)~|~\hat{h}\right] = 0
\end{align}
The diagonal entries of the co-variance conditional on the auxiliary dataset is:
\begin{multline*}
E[D^2~|~\hat{h}] = \frac{1}{n} \sum_{t\neq t'} E\left[\nabla_{\gamma} m(Z, \theta_0, h_0(X_t)) \cdot \left(\hat{h}(X) - h_0(X)\right)~|~\hat{h}\right]^2\\ + \frac{1}{n}\sum_{t=t'}  E\left[\|\nabla_{\gamma} m(Z, \theta_0, h_0(X)) \cdot \left(\hat{h}(X) - h_0(X)\right)\|^2~|~\hat{h}\right]
\end{multline*}
All the cross terms are zero by orthogonality, giving:
\begin{equation}
E[D^2~|~\hat{h}] =  E\left[\|\nabla_{\gamma} m(Z, \theta_0, h_0(X)))^2 \cdot \left(\hat{h}(X) - h_0(X)\right)\|^2\right]\leq \sigma^2 E\left[\|\hat{h}(X) - h_0(X)\|^2\right]
\end{equation}
Since $\hat{h}$ is consistent, we get that the latter converges to zero. Since the mean of $D$ and the trace of its co-variance converge to zero, we get that $D\rightarrow_p 0$. 
\end{proof}

Consistency of the estimator also follows easily from standard arguments, if one makes Assumption \ref{ass:main} and the extra condition that the moment condition in the limit is satisfied only for the true parameters, which is needed for identification (see e.g. \cite{Newey1994} for the formal set of extra regularity assumptions needed for consistency).

\section{Orthogonal Moments for Conditional Moment Problems}

One special case of when the orthogonality condition is satisfied is the following stronger, but easier to check property of conditional orthogonality:
\begin{defn}[Conditional Orthogonality] The moment conditions are conditionally orthogonal if:
\begin{equation}
\E\left[\nabla_\gamma m(Z, \theta_0, h_0(X)) | X\right] = 0
\end{equation} 
\end{defn}

\begin{lemma}
Conditional orthogonality implies orthogonality, when an auxiliary data set is used to estimate $\hat{h}$.
\end{lemma}
\begin{proof}
By the law of iterated expectations we have:
\begin{align*}
\E\left[\nabla_{\gamma} m(Z,\theta_0, h_0(X)) \cdot (\hat{h}(X)-h_0(X))\right] =~& \E\left[\E\left[\nabla_{\gamma} m(Z,\theta_0, h_0(X)) \cdot (\hat{h}(X)-\hat{h}(X))~|~\hat{h},X\right]\right]\\
=~& \E\left[\E\left[\nabla_{\gamma} m(Z,\theta_0, h_0(X))~|~\hat{h},X\right] \cdot (\hat{h}(X)-\hat{h}(X))\right] = 0
\end{align*}
Where in the last part we used the conditional orthogonality property.
\end{proof}

For conditional moment problems studied in \cite{Chamberlain1992}, \cite{Chernozhukov2016} shows how one can transform in an algorithmic manner an initial set of moments to a vector of orthogonal moments.

\bibliographystyle{alpha}
\bibliography{ortho}

\end{document}